\documentclass{article}
\usepackage{fullpage}
\usepackage{amsmath}
\usepackage{amsthm}
\usepackage{amsfonts}
\usepackage{graphicx}
\usepackage{algorithm}
\usepackage{algorithmic}
\newtheorem{proposition}{Proposition}
\DeclareMathOperator*{\argmax}{arg\,max}
\DeclareMathOperator*{\argmin}{arg\,min}

\begin{document}
\begin{center}
\huge
Climbing the Kaggle Leaderboard\\by Exploiting the Log-Loss Oracle\\
\normalsize
\vspace{.5cm}
Jacob Whitehill (\begin{tt}jrwhitehill@wpi.edu\end{tt})\\
Worcester Polytechnic Institute
\end{center}

\abstract{
In the context of data-mining competitions (e.g., Kaggle, KDDCup, ILSVRC 
Challenge \cite{ImageNet2015}), we show how access to an oracle that reports a contestant's
log-loss score on the test set can be exploited to deduce the ground-truth of some
of the test examples. By applying this technique
iteratively to batches of $m$ examples (for small $m$),
all of the test labels can eventually be inferred.  In this paper, 
(1) We demonstrate this attack on the first stage of 
a recent Kaggle competition (Intel \& MobileODT
Cancer Screening) and use it to achieve a log-loss of $0.00000$ (and
thus attain a rank of \#4 out of  848 contestants), without ever
training a classifier to solve the actual task. (2) We prove an upper bound
on the batch size $m$ as a function of the floating-point resolution of
the probability estimates that the contestant submits for the labels.
(3) We derive, and demonstrate in simulation, a more flexible attack that can be used even when
the oracle reports the accuracy on an unknown (but fixed) subset of the test set's labels.
These results underline the importance of evaluating contestants based only on test data that
the oracle does not examine.
}
\section{Introduction}
Data-mining competitions such as those as offered by Kaggle, KDD Cup, and other organizations,
have become a mainstay of machine learning.
By establishing common rules of participation as well as training and testing datasets that are
shared by all contestants, these competitions can help to advance the state-of-the-art of machine
learning practice in a variety of application domains.
In order for the scientific results of these contests to have value, however, it is imperative
that the methods by which candidates are evaluated be sound. The importance of fair evaluation is
made more pressing by the availability of oracles, often provided by the organizers of the competitions
themselves, that return the accuracy or loss value of the contestant's guesses with respect to the
test labels. The purpose of such oracles is to help participants to pursue more promising algorithmic
strategies and to improve the overall quality of contestants' submissions. But they
also open up the possibility of systematic overfitting, either inadvertently or maliciously.

In this paper, we consider how an oracle that returns the \emph{log loss} of a contestant's guesses
w.r.t.~the ground-truth labels of a test set, can be exploited by an attacker to infer the test set's
true labels. The log-loss is mathematically convenient because, unlike other metrics such as the 
AUC \cite{Tyler2000,AgarwalEtAl2005}, which is calculated over \emph{pairs} of examples, the log-loss 
can be computed for each example separately. Moreover, unlike the 0-1 loss that conveys only the number
of correctly labeled examples, the log-loss measures how ``close'' the contestant's guesses are to ground-truth.
The attack proposed in our paper can be effective despite limited floating-point resolution in the oracle's
return values, and can be applied even if the oracle only computes the log-loss on an unknown (but fixed)
subset of the test set. In a case study we performed for this paper,
we applied the attack to achieve a perfect score on a recent
Kaggle competition (Intel \& MobileODT Cervical Cancer Screening), thereby attaining a rank on the
first-stage competition leaderboard of \#4 out of 848. To be fair, Kaggle had structured their competition rules
such that the first stage was mostly for informational purposes to let contestants know how their
algorithmic approachs were faring compared to their contestants'. However, even a temporary high ranking
in a data-mining contest could conceivably hold ancillary value, e.g., by inducing
a potential employer or recruiter to take a look at a particular person's curriculum vitae.
In any case, the potential of exploiting a competition oracle underlines the importance
of employing commonsense safeguards to preserve the integrity of contestant rankings. In particular,
the results in this paper suggest that evaluation of contestants' performance should be done strictly on test examples on which
the oracle never reported accuracy.

\subsection{Related work}
Both intentional hacking \cite{Whitehill2016,blum2015ladder,zheng2015toward} and
inadvertent overfitting \cite{dwork2015preserving,hardt2014preventing} to test data in adaptive data analyses -- including
but not limited to data-mining competitions -- has generated recent research interest in the privacy-preserving
machine learning and computational complexity theory communities. 
Blum and Hardt \cite{blum2015ladder} recently described a ``boosting'' attack with which a contestant can estimate
the test labels such that, with probability $2/3$,  their accuracy w.r.t.~the ground truth
is better than chance. They also proposed a ``Ladder'' mechanism that can be used to rank contestants' performance and that
is robust to such attacks.
In addition, in our own prior work \cite{Whitehill2016}, we showed how an oracle that reports the AUC can be used
to infer the ground-truth of a few of the test labels with complete certainty.

\section{Notation and Assumptions}
We assume that the test set contains $n$ examples, each of which belongs to one of $c$ possible classes.
We represent the ground-truth of the entire test set using a row-wise 1-hot 
matrix ${\bf Y}_n \in \{0,1\}^{n \times c}$,
where ${\bf Y}_n \doteq [{\bf y}_1,\ldots,{\bf y}_n]^\top$, each ${\bf y}_i = [ y_{i1}, \ldots, y_{ic} ]^\top$, each
$y_{ij} \in \{0,1\}$, and $\sum_j y_{ij}=1$ for each $i$.
Similarly, we represent the contestant's guesses using matrix $\widehat{\bf Y}_n\in\mathbb{R}^{n\times c}$, where
$\widehat{\bf Y}_n \doteq [\widehat{\bf y}_1,\ldots,\widehat{\bf y}_n]^\top$, each $\widehat{\bf y}_i = [ \widehat{y}_{i1}, \ldots, \widehat{y}_{ic} ]$,
each $\widehat{y}_{ij} \in (0,1)$, and $\sum_j \widehat{y}_{ij}=1$ for each $i$.
The loss function that
we study in this paper is the \emph{log-loss}, computed by function $f$, of the contestant's guesses with respect to the ground-truth:
\begin{equation}
\label{eqn:log_loss}
\ell_n = f({\bf Y}_n , \widehat{\bf Y}_n ) = - \frac{1}{n} \sum_{i=1}^n \sum_{j=1}^c y_{ij} \log \widehat{y}_{ij}
\end{equation}
We assume that the data-mining competition offers an oracle to which a contestant can submit her/his real-valued
guesses for the test labels and obtain the log-loss $\ell_n$ of
the guesses with respect to the ground-truth.

\section{Example}
To show how knowledge of the log-loss can reveal information about the ground-truth itself,
consider a tiny test set in which 
there are just $2$ examples and $3$ classes, and suppose a contestant submits the following guesses to the oracle
(where $e$ is the base of the natural logarithm):
\[\widehat{\bf Y}_2 = \left[ \begin{array}{ccc}
  e^{-2}\ & e^{-1}\ & 1-e^{-2}-e^{-1} \\
  e^{-8}\ & e^{-4}\ & 1-e^{-8}-e^{-4} \\
\end{array} \right]
\]
If the oracle reports that the log-loss 
of the contestant's guesses, with respect to the ground-truth, is $3$,
then the ground-truth labeling of these two examples \emph{must} be 
\[{\bf Y}_2 = \left[\begin{array}{ccc}
  1 & 0 & 0 \\
  0 & 1 & 0 \\
\end{array} \right]
\]
since this is the only value of ${\bf Y}_2$  that satisfies
\[
f({\bf Y}_2, \widehat{\bf Y}_2) = - \frac{1}{2} \sum_{i=1}^2 \sum_{j=1}^3 y_{ij} \log \widehat{y}_{ij} = 3
\]

\section{Problem formulation and proposed solution}
Here we describe an attack (summarized
in Algorithm \ref{alg:cheat})
with which a contestant can iteratively (in batches) infer the ground-truth of the test
set while simultaneously improving (usually) her/his standing on the competition leaderboard.
For simplicity of notation (and without loss of generality), we assume that the examples in the test set
are ordered such that
(1) the labels of the first $k$ examples have already been inferred; (2) the labels of the next $m$ examples (which we call the \emph{probed} examples) are to be inferred in the current round; and (3) the rest ($n-m-k$) of the test labels will 
uninferred during the current round.

Since $f$ is summed over all $n$ examples, it can be re-written as
\begin{equation}
\label{eqn:loss}
\ell_n \doteq f({\bf Y}_n , \widehat{\bf Y}_n ) = - \frac{1}{n}\left[
                                             \sum_{i=1}^{k} \sum_j y_{ij} \log \widehat{y}_{ij} +
                                             \sum_{i=k+1}^{k+m} \sum_j y_{ij} \log \widehat{y}_{ij} +
                                             \sum_{i=k+m+1}^{n} \sum_j y_{ij} \log \widehat{y}_{ij} \right]
\end{equation}
Assuming the first $k$ examples have already been inferred correctly so that $\widehat{y}_{ij}=y_{ij}$ for each $j$ and each
$i\leq k$, then the first term in the RHS of Equation \ref{eqn:loss}
is approximately $0$.\footnote{In practice, competition oracles often enforce that each $\widehat{y}_{ij} \in
[\gamma, 1 - \gamma]$ where is a small number such as $\gamma=1\times 10^{-15}$;
this results in a negligible cost of $k \log (1 - (c-1) \gamma)$ for the second term. See Section
\ref{sec:practical_limit}.}
Moreover, if we set $\widehat{y}_{ij}$ to $1/c$ for each $j$ and each $i=k+m+1,\ldots,n$, then the
third term in the RHS equals $-1 \times (n-m-k) \times \log c$. The log-loss $\ell_m$ due to just the $m$
probed examples is thus:
\begin{eqnarray}
\ell_m &\doteq& \frac{1}{m} \sum_{i=k+1}^{k+m} \sum_j y_{ij} \log \widehat{y}_{ij} \nonumber \\ 
       &=& \frac{1}{m}\left((n-m-k)\log c - n \times \ell_n \right) \label{eqn:partial_loss}
\end{eqnarray}

Given knowledge of $\ell_m$ -- which can be calculated from the value $\ell_n$ returned by the oracle -- and
given knowledge of $[ \widehat{\bf y}_{k+1}, \ldots, \widehat{\bf y}_{k+m} ]^\top$ -- which
the candidate her/himself controls -- the ground-truth of the $m$ examples can be inferred using exhaustive
search over all possible $c^m$ labelings (for small $m$).
Over $\lceil n/m \rceil$ consecutive rounds, we can determine the ground-truth values of \emph{all} of the
test labels.
\begin{algorithm}
\caption{Infer the ground-truth of the test set using the oracle's response $\ell_n$.}
\label{alg:cheat}
\begin{algorithmic}
\REQUIRE A probe matrix ${\bf G}_m$, where $m$ is the number of probed examples in each round. 
\REQUIRE An oracle that reports $f({\bf Y}_n, \widehat{\bf Y}_n)$.
\ENSURE The ground-truth labels ${\bf Y}_n$ for all $n$ examples.
\FOR{round $r=1,\ldots,\lceil n/m \rceil$}
	\STATE 1. Set $k\leftarrow (r-1)\times m$.
	\STATE 2. Configure the probe matrix $\widehat{\bf Y}_n$:
	\begin{itemize}
	\item For $i=1,\ldots,k$ (examples that have already been inferred), set $\widehat{\bf y}_i$ to the inferred ${\bf y}_i$.
	\item For $i=k+1,\ldots,k+m$ (the probed examples), set $\widehat{\bf y}_i$ to the corresponding row of ${\bf G}_m$.
	\item For $i=k+m+1,\ldots,n$ (examples that will remain uninferred), set $\widehat{y}_{ij}$ to $1/c$ for all $j$.
	\end{itemize}
	\STATE 3. Submit $\widehat{\bf Y}_n$ to oracle and obtain $\ell_n$.
	\STATE 4. Compute $\ell_m$ (the loss on just the $m$ probed examples) according to Equation \ref{eqn:partial_loss}.
	\STATE 5. Determine the ground-truth of the probed examples by finding $[ {\bf y}_{k+1}, \ldots, {\bf y}_{k+m} ]^\top$ that
	minimizes $\epsilon$.
\ENDFOR
\end{algorithmic}
\end{algorithm}

Note that, due to the finite floating-point resolution of the oracle's return value,
there is usually a small difference between the $\ell_m$ that is calculated based on the oracle's response
and the true log-loss value of the $m$ probed examples. We call this difference the \emph{estimation error}:
\[
\epsilon \doteq |\ell_m - f([ \widehat{\bf y}_{k+1}, \ldots, \widehat{\bf y}_{k+m} ]^\top, [ {\bf y}_{k+1}, \ldots, {\bf y}_{k+m} ]^\top)|
\]

\subsection{How to choose the guesses of the probed examples}
The key to exploiting the oracle so as to infer the ground-truth correctly
is to choose the guesses $\widehat{\bf Y}_n$ so that $\ell_m$ reveals the labels of the probed examples uniquely.
To simplify notation slightly, we let
${\bf G}_m \doteq [ \widehat{\bf y}_{k+1}, \ldots, \widehat{\bf y}_{k+m} ]^\top$ -- which we call the \emph{probe matrix}
-- represent the contestant's guesses for the $m$ probed examples. The probe matrix can stay the same across all submission rounds.

If the contestant's guesses were real numbers in the mathematical sense -- i.e., with
infinite decimal resolution -- then ${\bf G}_m$ could be set to random values, constrained
so that each row sums to 1. With probability
$1$, the log-loss $\ell_m$ would then be unique over all possible instantiations of $[ {\bf y}_{k+1},\ldots,{\bf y}_{k+m} ]^\top$. However,
in practice, both the guesses and the log-loss reported by the oracle have finite precision, and ``collisions'' -- different
values of the ground-truth that give rise to the same, or very similar, losses -- could occur.
Consider, for example, the probe matrix below:
\[
\left[ \begin{array}{ccc}
$0.53595382$ & $0.20743777$ & $0.25660840$ \\
$0.76336402$ & $0.17982958$ & $0.05680643$ \\
$0.83539897$ & $0.02825473$ & $0.13634628$ \\
$0.88845736$ & $0.10858667$ & $0.00295598$ \\
                   \end{array}\right]
\]
Here, no two elements are closer than $0.025$ apart. Nonetheless, two possible values
for the ground-truth labeling ${\bf Y}_4$ result in log-loss values (approximately $1.18479$ and $1.18488$, respectively)
that are less than $10^{-4}$ apart; these candidate labelings are
\begin{eqnarray*}
\left[ \begin{array}{ccc}
$0$ & $1$ & $0$ \\
$0$ & $0$ & $1$ \\
$1$ & $0$ & $0$ \\
$1$ & $0$ & $0$ \\
\end{array}\right]
\quad
\textrm{and}
\quad
\left[ \begin{array}{ccc}
$1$ & $0$ & $0$ \\
$0$ & $1$ & $0$ \\
$1$ & $0$ & $0$ \\
$0$ & $1$ & $0$ \\
\end{array}\right]
\end{eqnarray*}
If the oracle returned a log-loss of, say, $1.185$, then it would be ambiguous which of the two values of
${\bf Y}_4$ was the correct one.

\subsection{Collision avoidance}
\label{sec:collisions}
In order to avoid collisions, we need to choose ${\bf G}_m$ so that the minimum distance -- over all
possible pairs of different ground-truth labelings of the $m$ examples --
is large enough so that even a floating-point approximation of $\ell_m$
can uniquely identify the ground-truth. We can thus formulate a constrained optimization problem
in which we express the \emph{quality} $Q$ of ${\bf G}_m$ as:
\begin{equation}
\label{eqn:quality}
Q({\bf G}_m) \doteq \min_{{\bf Y}_m \ne {\bf Y}_m'} |f({\bf Y}_m,{\bf G}_m) - f({\bf Y}_m',{\bf G}_m)|
\end{equation}
where ${\bf Y}_m,{\bf Y}_m'$ are \emph{distinct} ground-truth labelings, and we wish to find
\[
{\bf G}_m^* \doteq \argmax_{{\bf G}_m} Q({\bf G}_m)
\]
subject to the constraints that each row of ${\bf G}_m$ be a probability distribution.
Optimization algorithms do exist to solve constrained minimax and maximin problems
\cite{madsen1978linearly,brayton1979new}; however,
in practice, we encountered numerical difficulties when using them (specifically, the MATLAB implementation in
\begin{tt}fminimax\end{tt}), whereby the constraints at the end of optimization were not satisfied.
Instead, we resorted to a heuristic that  is designed
to sample the guesses $\widehat{y}_{ij}$ so that the sum of the logarithms of a randomly chosen subset of the guesses --
one for each example $i$ -- is far apart within the range of 32-bit floating-point numbers. In particular, we set
$\widehat{y}_{ij}=a\times 10^{b}$ and sampled $a \sim \mathcal{U}([0,1])$
and $b \sim \mathcal{U}(\{-14,13,\ldots,-1,0\})$ for each $i$ and each $j<c$. For $j=c$, we sampled
$a \sim \mathcal{U}([0,1])$ and $b$ was fixed to $0$. Finally, we normalized each $\widehat{\bf y}_i$
so that the entries sum to 1.
Based on this heuristic, we used Monte-Carlo sampling (with approximately $10000$ samples) to
optimize the maximin expression above. In particular, for $m=6$, we obtained a matrix
${\bf G}_6$ for which 
$Q({\bf G}_6)= 0.00152$. This number is substantially greater than the largest estimation error we ever encountered
during our attacks (see Section \ref{sec:experiment_kaggle})
and thus enabled us to conduct our attack in batches of $6$ probed excamples.
However, for $m=7$, we were never able to find a ${\bf G}_7$ for which the
RHS of Equation \ref{eqn:quality} above exceeded $0.0001$.

\section{Practical limit on number of probed examples $m$}
\label{sec:practical_limit}
The oracles in data-mining competitions such as Kaggle often impose
a limit on the submitted probabilities so that
$\widehat{y}_{ij}\in[\gamma, 1-\gamma]$ for each $i,j$, where $\gamma$ is a small number such as 
$10^{-15}$. This ensures that the log-loss is well defined (so that $\log 0$ is never evaluated) but also
indirectly imposes a limit on how many examples $m$ can be probed during each round. In particular, we can
prove an upper bound on the quality of a probe matrix ${\bf Q}_m$ as a function of $\gamma$:

\begin{proposition}
Let $m$ be the number of probed examples and let $c$ be the number of possible classes. 
Let $\gamma\in (0,1)$ represent the minimum value, imposed by the oracle, of any guess $\widehat{y}_{ij}$.
Then the quality $Q({\bf G}_m)$ of any probe matrix ${\bf G}_m$ is bounded above
by $\frac{\log (1-(c-1)\gamma) - \log \gamma}{c^m - 1}$.
\end{proposition}
\begin{proof}
Since $\gamma$ is the minimum value of any element in ${\bf G}_m$, then $1-(c-1)\gamma$ is the maximum value.
Each of the $m$ probed examples must therefore contribute at least $- \log (1-(c-1)\gamma)$
and at most $- \log \gamma$
to the log-loss. Averaged over all $m$ examples, the log-loss must therefore be in the closed interval
\[ I \doteq \left[ - \frac{1}{m} \sum_i \log (1 - (c-1)\gamma),\quad - \frac{1}{m} \sum_i \log \gamma \right] =
            \left[ - \log (1 - (c-1)\gamma), - \log \gamma \right] \]
Since there are $c$ classes, then there are $c^m$ possible ground-truth labelings and corresponding
log-losses. The maximium value of  $Q({\bf G}_m)$ -- i.e., the minimum
distance, over all possible ground-truth labelings, between corresponding log-loss values -- 
is attained when the log-losses are distributed across $I$ so that the $c^m-1$ ``gaps'' between
consecutive pairs of log-loss values  are equal in size. Therefore, the maximum value of $Q({\bf G}_m)$ is at most
\[
\delta = \frac{\log (1-(c-1)\gamma) - \log \gamma}{c^m - 1}
\]
\end{proof}
Since $\delta$ decreases exponentially in $m$,  and since $\delta$ must be kept larger than
the maximum estimation error $\epsilon$ observed when executing Algorithm \ref{alg:cheat}, then 
$m$ must necessarily be kept small.  As an example, for $m=10$, $\delta \approx 5.8 \times 10^{-4}$, and in practice we were not able
to find satisfactory ${\bf G}_m$ even  for $m\geq 7$. Nonetheless, even with $m=6$, we were able to climb the leaderboard 
of a recent Kaggle competition successfully.

\section{Experiment: Kaggle Competition}
\label{sec:experiment_kaggle}
We tested Algorithm \ref{alg:cheat} on the Intel \& MobileODT Cervical Cancer Screening competition hosted by Kaggle in May-June 2017.
The objective of the competition was to develop an automatic classifier to analyze cervical scans of women who are at-risk
for cervical cancer and to predict the most effective treatment based on the scan. Such a classifier could potentially save many lives, especially
in rural parts of the world in which high-quality medical care is lacking. During the first stage of the contest,
the competition website provided each contestant with
training images ($1821$) and associated training labels (with $c=3$ categories),
as well as testing images ($n=512$). The goal of the competition was to predict
the test labels with high accuracy. To help competitors identify the most promising classification methods, Kaggle provided
an oracle -- which each contestant could query up to 5 times per day -- that reported the log-loss on \emph{all} $512$ test
examples \emph{without} any added noise. After the first-stage submission deadline (June 14, 2017), the second  stage of the competition 
began, using a larger test set and an oracle that reported the loss  on only a fixed subset of the test samples.

{\bf Cheating during the first stage}:  Since the oracle during the first stage of the  competition  returned the log-loss on the
\emph{entire} test set, it provided an ideal environment in which to demonstrate Algorithm \ref{alg:cheat}. We performed the attack
in phases according to the following procedure:
For the first 2 queries, we probed only a single test label (i.e., $m=1$) just to verify that our code was working correctly.
For the next 30 queries, we probed $m=4$
labels (using the ${\bf G}_4$ shown in the appendix).
For the remaining queries (after we had found ${\bf G}_6$ with large enough $Q$), we probed $m=6$ examples per query.
The maximum estimation error $\epsilon$, over all rounds, between the log-loss returned by the oracle and the loss calculated
based on the inferred
ground-truth, was less than $0.0061$. This was less than half 
of $Q({\bf G}_6)$ for the probe matrix we used and thus allowed us to infer the ground-truth unambiguously.
During the competition we did not perform any supervised learning of cervical scan images 
(i.e., the intended purpose of the competition) whatsoever.

{\bf Results}: The progression of our attack is shown in Figure \ref{fig:climbing}. In short, with less than 100
oracle queries (well within the limit given the total duration of the competition),
we were able to infer the ground-truth labels of all the test examples perfectly.
We note that, during consecutive iterations of the attack, the attained log-loss need not always decrease -- this is because
the the reduction in loss due to inferring more examples can sometimes be dwarfed by an increase in loss due to which specific
entries of ${\bf G}_m$ are selected by the corresponding ground-truth of the probed examples.
Nevertheless, the proposed attack  is able to recover perfectly all $m$ probed labels during every round. The progression of (usually
decreasing) log-loss values $\ell_n$ are plotted in Figure \ref{fig:climbing} (left). By the last iteration, we had
recovered the ground-truth
values of all $512$ examples correctly and thus attained a loss of $0.00000$. Since we were tied with several other contestants who 
also achieved the same loss -- whether by legitimate means or by cheating -- we were ranked in $4$th place on the first-stage leaderboard
(see Figure \ref{fig:climbing} (right)).

{\bf Second stage}: During the second stage of the same Kaggle competition, the organizers created a larger test set 
that included the $512$ test examples from the first-stage as a subset. Moreover, these same $512$ examples
were the basis of both the oracle results and the leaderboard rankings  up until the conclusion of the second-stage competition.
Hence, for a brief period of about two weeks, we were able to maintain the illusion
of a top-ranked Kaggle competitor achieving a perfect score. To be clear: the \emph{final}, \emph{definitive} results of the competition
(announced on June 21, 2017) -- including who won the \$100,000 prize money --  were based on the log-loss on the \emph{entire} test set, not just
the subset. Naturally, our ranking declined precipitously at this point (to 225th place out of 848 contenders) since our guesses
on the remaining $75\%$ examples were just $1/c$.

\begin{figure}
\begin{center}
\includegraphics[width=3.3in]{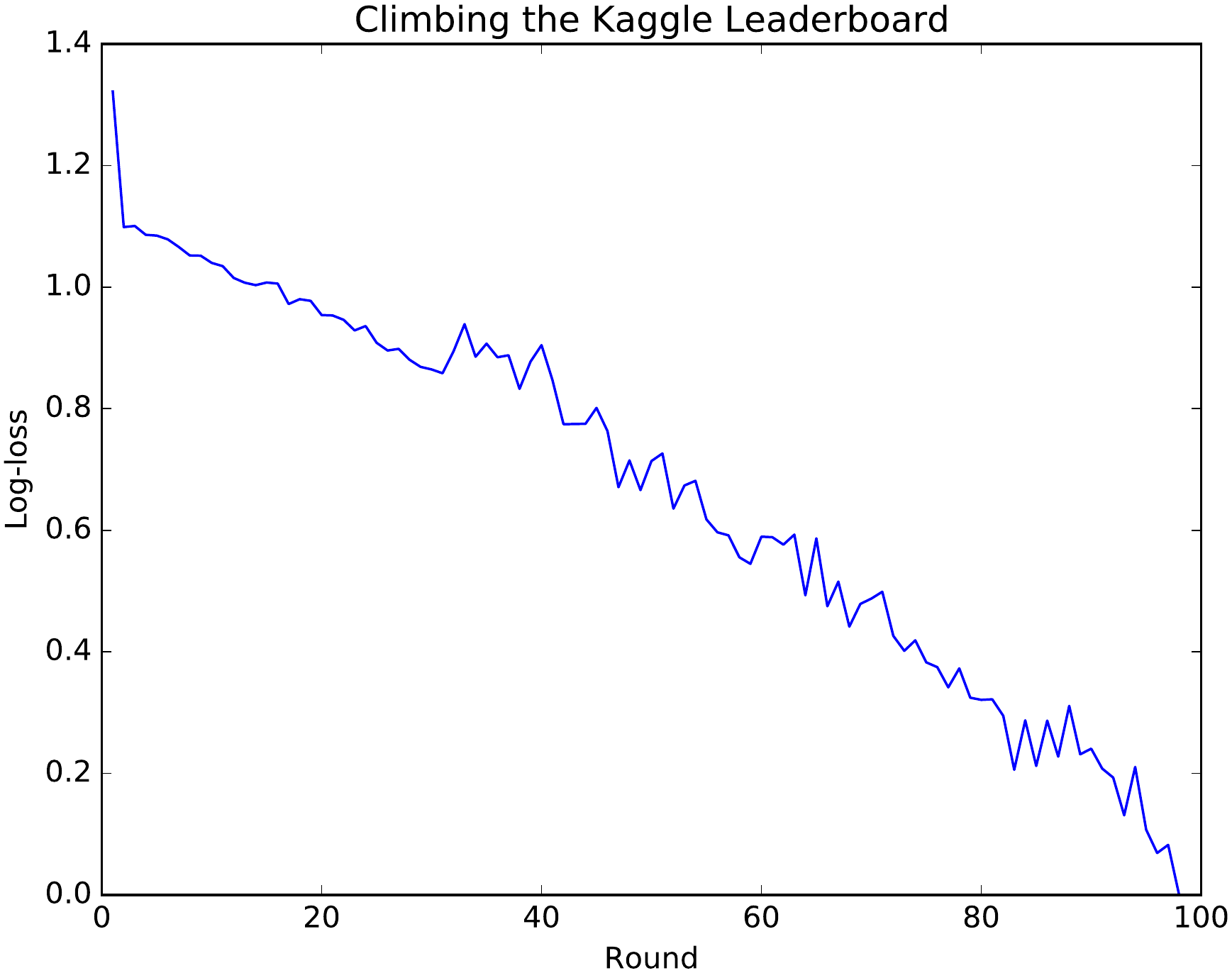}
\raisebox{+.15\height}{\includegraphics[width=3in]{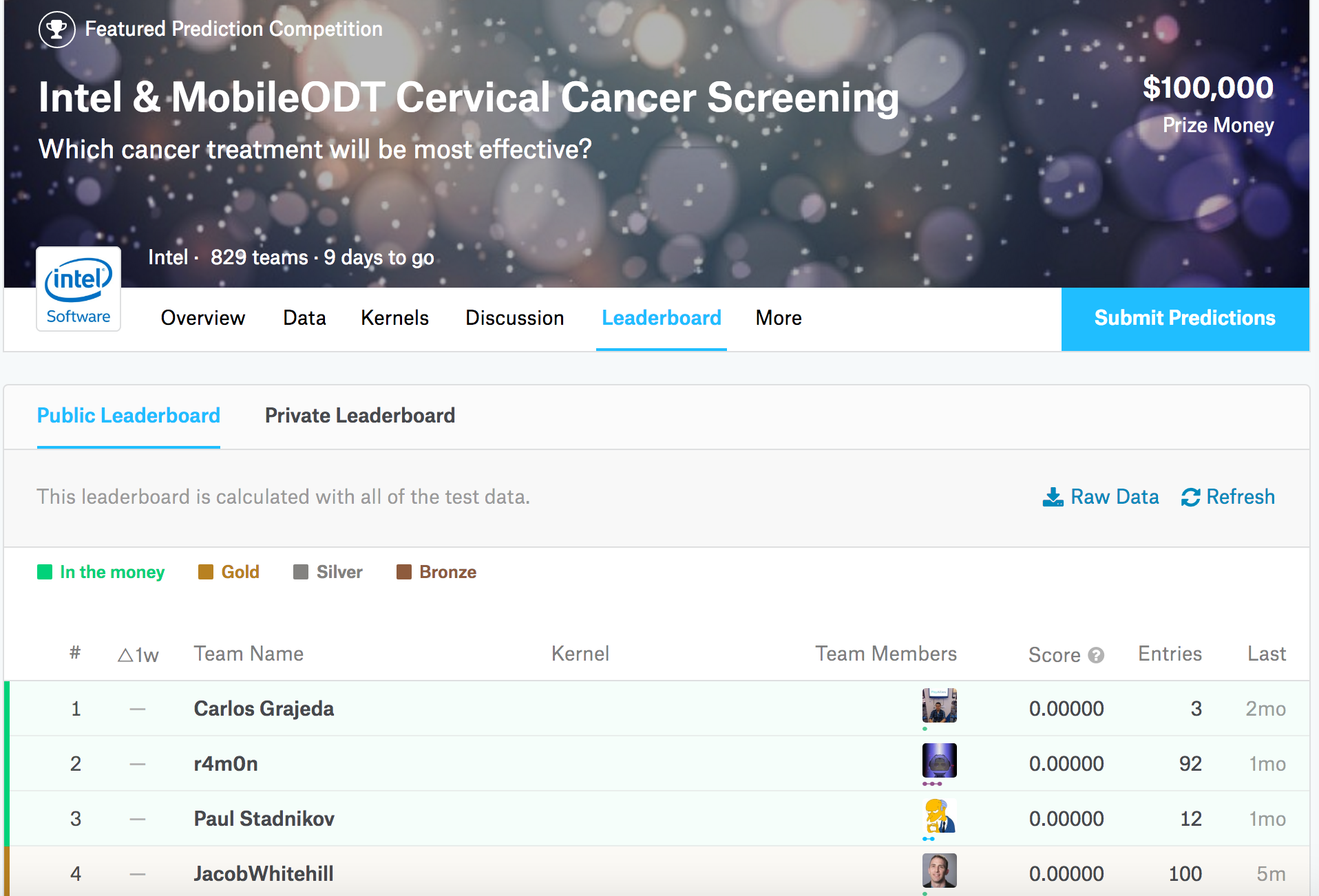}}
\caption{Climbing the Kaggle leaderboard, using the proposed log-loss oracle exploitation algorithm,
of the Intel \& MobileODT Cervical Cancer Screening 2017 competition (first stage). During the first $32$ oracle
queries, the labels of $m=4$ test examples  were ascertained during each round. Afterwards, a more aggressive
($m=6$) approach was used to descend the log-loss curve more quickly. After 98 rounds of Algorithm \ref{alg:cheat}, we were able to infer
all $512$ test labels correctly, achieved a loss of $0.00000$, and climbed the leaderboard to rank \#4 out of 848
(for the first stage of the competition).}
\label{fig:climbing}
\end{center}
\end{figure}

\section{Cheating when the Oracle Reports Accuracy on a Subset of Examples}
\label{sec:cheat_over_subset}
It is more common in data-mining competitions for the oracle to report accuracy only on
a \emph{subset} of the test set. Here we describe how a contestant can still cheat, using similar methods
as described above, when the oracle reports accuracy on a \emph{fixed subset} of examples (i.e., the
same subset for each oracle query). In this setting, the contestant submits a matrix $\widehat{\bf Y}_n$  with $n$ rows,
but the log-loss obtained from the oracle is based on only $s\leq n$ examples. Note that, in contrast to Algorithm \ref{alg:cheat},
here we treat ``already inferred'' examples in the same way as the ``uninferred'' examples -- we assign their guesses
$\widehat{y}_{ij}=1/c$ for all $i,j$ (instead of setting them to their inferred values). The only drawback of this simplification
is that the attacker cannot simultaneously infer the ground-truth \emph{and} decrease her/his log-loss 
(in the manner illustrated by Figure \ref{fig:climbing}) -- rather, the contestant must wait until after she/he has
inferred the ground-truth to ``cash in'' and jump to a higher leaderboard rank. We describe the new attack below:

{\bf Determining the size of the subset $s$}:
The first step of the attack is to determine the value of $s$. 
To this end, it is useful to identify a \emph{single} test example
that is definitely in the $s$-element subset on which the oracle
reports accuracy. Finding such an example can be achieved
by setting the guesses $\widehat{y}_{ij}$ to $1/c$ for all but one example and setting
the guesses to random values (but not equal to $1/c$) for a single ``probe'' example $i$.
If the loss reported by the oracle is \emph{not} equal to $- \log c$, then the probe example must be one of the $s$
evaluated examples; otherwise,
another example is chosen and the procedure is repeated.
Assuming that the fraction $s/n$ is not too small, then this procedure should only 
take a few oracle queries.

Given a single example at index $i$ that is known to be among the $s$ evaluated examples,
along with the log-loss value $\ell_s$, the contestant can  determine $s$. To see how, notice that 
the $s-1$ examples that are \emph{not} example $i$ contribute a log-loss of $\frac{s-1}{s} \log c$, and that 
example $i$ contributes $- \frac{1}{s} y_{ij} \log \widehat{y}_{ij}$.
Therefore, the contestant can iterate (jointly) over all 
$n$ possible values for $s \in \{1,\ldots, n\}$ \emph{and} all $c$ possible values of ${\bf y}_i$ to find
\begin{equation}
\label{eqn:infer_s}
\argmin_s \left\{ \min_{{\bf y}_i} \left|\frac{1}{s}\left( (s-1)\times \log c - \sum_j y_{ij} \log \widehat{y}_{ij}\right) - \ell_s\right| \right\}
\end{equation}
The solution is the number of examples $s$ on which the oracle reports the log-loss. We note that, in practice,
due to finite resolution of the oracle's response and the contestant's guesses, the inferred value of $s$ can sometimes
be inaccurate. Nevertheless, even an imperfect estimate of $s$ can often be used to infer the ground-truth of the
$s$ evaluated test examples with high accuracy (see simulation results below).

{\bf Inferring the labels of a batch of probe examples}:
Now that $s$ has been inferred, the contestant can probe the labels of $m$ examples at a time.  To infer ${\bf Y}_m$, the contestant
must consider whether each probed example $i$ is in the $s$-element subset on which the oracle reports the log-loss. To this end, we
define ${\bf z} \in \{0,1\}^m$ so that $z_i$ is $1$ if example $i$ is in the $s$-element subset and $0$ otherwise. The $L_1$-norm
$\|{\bf z}\|_1$ of this vector thus equals the number of probed examples that are also in the $s$-element subset, and we can compute
the contribution of the probed examples to the log-loss as $-\frac{1}{s} \sum_i \sum_j z_i y_{ij} \log \widehat{y}_{ij}$.
The remaining log-loss is accounted for by the $(s-\|{\bf z}\|_1)$ ``uninferred'' examples and amounts to $\frac{s-\|{\bf z}\|_1}{s} \log c$. Therefore,
to determine ${\bf Y}_m$ and ${\bf z}$, the contestant must optimize
\begin{equation}
\label{eqn:cheat_over_subset}
\argmin_{{\bf Y}_m} \left\{ \min_{{\bf z}\in \{0,1\}^m} \left| \left( -\frac{1}{s} \sum_i \sum_j z_i y_{ij} \log \widehat{y}_{ij} +
                                         \frac{s-\|{\bf z}\|_1}{s} \log c\right) - \ell_n \right| \right\}
\end{equation}
using brute-force search (which is easy since $m$ is small). Naturally, row $i$ of the inferred matrix ${\bf Y}_m$ is valid only if $z_i=1$.

{\bf Choosing the guesses}:  Similar to Section \ref{sec:collisions}, we need to optimize the probe matrix
so as to to minimize collisions.
In this setting, however, we must be concerned not just with different possible ground-truth labelings ${\bf Y}_m$, but also
with the indicator variables ${\bf z}$ -- both of which ``select'' different elements $\widehat{y}_{ij}$ to add to the log-loss.
We thus revise the quality function to maximize the minimum distance, over all distinct pairs ${\bf Y}_m \ne {\bf Y}_m'$ 
\emph{and} all all distinct pairs ${\bf z} \ne {\bf z}'$, of the corresponding log-loss values (see Appendix).

\subsection{Simulation}
Instead of applying the algorithm above to the Intel-MobileODT competition\footnote{
The oracle in the second stage of the Intel-MobileODT 2017 competition evaluated only a subset
of the test examples, but it turned out that this subset was exactly the $512$ images from the first-stage test set;
hence, there was nothing new to infer.}, we conducted a simulation. In particular,
we simulated a test set  containing $n=2048$ examples (from $c=3$ classes)
where $s=512$ evaluated examples was randomly sampled (but fixed
over all oracle queries) across the entire test set. The simulated contestant first probed just single examples until it
could infer $s$ (using Equation \ref{eqn:infer_s}). The contestant then proceeded
to submit batches of $m=4$ probed examples (for $\lceil n/m \rceil$ total rounds)
and infer their labels based on the oracle's response by optimizing
Equation \ref{eqn:cheat_over_subset} using exhaustive search.

To assess how the floating-point resolution $p$ of the oracle's response impacts the accuracy of the labels inferred
by the contestant, we varied $p\in \{1,2,3,4,5\}$, where $p$ was the number of digits after the decimal point.
($p=1$ means that the oracle's responses were rounded to the nearest $0.1$; $p=2$ to the nearest $0.01$, etc.)
For each $p$ value, we conducted 100 simulations. As the probe matrix for all simulations, we used $\widetilde{G}_4$ (see appendix).
At the end of each simulation, we computed the accuracy of the inferred labels versus ground-truth,
and then averaged the accuracy rate over all 100 simulations.

{\bf Results}:
Accuracy of inferred labels increased with higher floating-point resolution, as expected.
For $p=1$, the mean accuracy was $38.8\%$, which was still slightly higher than
the expected baseline guess accuracy ($\approx 33.33\%$).
For $p=2$, mean accuracy was $47.8\%$; for  $p=3$, $59.4\%$; for $p=4$, $78.7\%$;
and for $p=5$, $93.6\%$. In many simulations, the contestant's best inference of $s$ -- the exact number of evaluated
examples in the test set that is needed in Equation \ref{eqn:cheat_over_subset} -- was incorrect,
and yet many (and often most) of the inferred test labels were still correct.
In fact, the average correlation (over all values of $p$) between the mean absolute error between the inferred $s$ and its
true value, and the accuracy of the inferred test labels w.r.t.~ ground-truth, was only $-.107$ -- suggesting that 
correct inference of the test labels is relatively robust to errors in inference of $s$.

\section{Conclusion}
We derived an algorithm whereby a contestant can illicitly improve her/his leaderboard score in a data-mining
competition by exploiting information provided by an oracle that reports the log-loss of the contestant's guesses
w.r.t.~the ground-truth labels of the test set. We also showed that the number $m$ of test examples
whose labels can be inferred in each round of the attack is fundamentally
limited by the floating-point resolution of the contestant's guesses. Nevertheless, the attack is practical, and
we demonstrated it on a recent
Kaggle competition and thereby attained a leaderboard ranking of \#4 out of 848 (for the first stage of the contest), without
ever even downloading the training or testing data.
For more general scenarios in which the oracle reports the log-loss on only a fixed subset of test examples,
we derived a second algorithm and demonstrated it in simulation. In terms of {\bf practical implications},
our findings suggest that data-mining competitions
should evaluate contestants based only on test examples that the loss/accuracy oracle never examined.

\section*{Appendix: Probe Matrices ${\bf G}_m$}
For the experiment in Section \ref{sec:experiment_kaggle}, we used 
the matrices ${\bf G}_6$ and ${\bf G}_4$ shown below (note that $\textrm{e}$ here
means ``times 10 to the power\ldots''):
\begin{eqnarray*}
{\bf G}_4 &=& \left[ \begin{array}{ccc}
$3.17090802e-01$ & $6.03843391e-01$ & $7.90658068e-02$ \\
$3.34653412e-01$ & $6.64893789e-01$ & $4.52799011e-04$ \\
$4.44242183e-01$ & $5.42742523e-01$ & $1.30152938e-02$ \\
$3.02254057e-01$ & $1.41415552e-01$ & $5.56330391e-01$ \\
                   \end{array}\right] \qquad \textrm{and}\quad Q({\bf G}_4)=0.019248
\\
{\bf G}_6 &=& \left[ \begin{array}{ccc}
$3.72716316e-13$ & $3.17270110e-06$ & $9.99996841e-01$ \\
$4.03777185e-11$ & $2.98306441e-06$ & $9.99997020e-01$ \\
$1.51235222e-11$ & $9.45069790e-02$ & $9.05493021e-01$ \\
$7.54659835e-10$ & $6.77224932e-07$ & $9.99999344e-01$ \\
$1.84318694e-09$ & $2.37398371e-01$ & $7.62601614e-01$ \\
$9.75336131e-12$ & $1.44393380e-06$ & $9.99998569e-01$ \\
                   \end{array}\right] \qquad \textrm{and}\quad Q({\bf G}_6)=0.001526
\end{eqnarray*}
For the simulation in Section \ref{sec:cheat_over_subset}, we defined a new quality function
\[
\widetilde{Q}(\widetilde{\bf G}_m) \doteq 
\min_{\substack{{\bf z}_m,\\{\bf Y}_m \ne {\bf Y}_m'}}
  |\widetilde{f}({\bf z}_m, {\bf Y}_m,{\bf G}_m) - \widetilde{f}({\bf z}_m, {\bf Y}_m',{\bf G}_m)|
\]
where
\[
\widetilde{f}({\bf z}_m, {\bf Y}_m,\widehat{\bf Y}_m) = -\frac{1}{m} \sum_i \sum_j z_i y_{ij} \log \widehat{y}_{ij}
\]

In our simulated attack against the oracle, we used:
\begin{eqnarray*}
\widetilde{\bf G}_4 &=& \left[ \begin{array}{ccc}
$3.34296189e-02$ & $6.06806998e-06$ & $9.66564298e-01$ \\
$6.80901580e-15$ & $8.52564275e-02$ & $9.14743602e-01$ \\
$1.78242549e-01$ & $2.03901175e-12$ & $8.21757436e-01$ \\
$1.22676250e-02$ & $1.40922994e-03$ & $9.86323118e-01$ \\
                   \end{array}\right] \qquad \textrm{and}\quad \widetilde{Q}(\widetilde{{\bf G}}_4)= 0.012750
\end{eqnarray*}

\bibliographystyle{ieee}
\bibliography{paper}

\begin{thebibliography}{10}\itemsep=-1pt

\bibitem{AgarwalEtAl2005}
S.~Agarwal, T.~Graepel, R.~Herbrich, S.~Har-Peled, and D.~Roth.
\newblock Generalization bounds for the area under the {ROC} curve.
\newblock In {\em Journal of Machine Learning Research}, pages 393--425, 2005.

\bibitem{blum2015ladder}
A.~Blum and M.~Hardt.
\newblock The ladder: A reliable leaderboard for machine learning competitions.
\newblock {\em arXiv preprint arXiv:1502.04585}, 2015.

\bibitem{brayton1979new}
R.~Brayton, S.~Director, G.~Hachtel, and L.~Vidigal.
\newblock A new algorithm for statistical circuit design based on quasi-newton
  methods and function splitting.
\newblock {\em IEEE Transactions on Circuits and Systems}, 26(9):784--794,
  1979.

\bibitem{dwork2015preserving}
C.~Dwork, V.~Feldman, M.~Hardt, T.~Pitassi, O.~Reingold, and A.~L. Roth.
\newblock Preserving statistical validity in adaptive data analysis.
\newblock In {\em Proceedings of the Forty-Seventh Annual ACM on Symposium on
  Theory of Computing}, pages 117--126. ACM, 2015.

\bibitem{hardt2014preventing}
M.~Hardt and J.~Ullman.
\newblock Preventing false discovery in interactive data analysis is hard.
\newblock In {\em Foundations of Computer Science (FOCS), 2014 IEEE 55th Annual
  Symposium on}, pages 454--463. IEEE, 2014.

\bibitem{madsen1978linearly}
K.~Madsen and H.~Schj{\ae}r-Jacobsen.
\newblock Linearly constrained minimax optimization.
\newblock {\em Mathematical Programming}, 14(1):208--223, 1978.

\bibitem{ImageNet2015}
O.~Russakovsky, J.~Deng, H.~Su, J.~Krause, S.~Satheesh, S.~Ma, Z.~Huang,
  A.~Karpathy, A.~Khosla, M.~Bernstein, A.~C. Berg, and L.~Fei-Fei.
\newblock {ImageNet Large Scale Visual Recognition Challenge}.
\newblock {\em International Journal of Computer Vision (IJCV)}, pages 1--42,
  April 2015.

\bibitem{Tyler2000}
C.~Tyler and C.-C. Chen.
\newblock Signal detection theory in the 2{AFC} paradigm: attention, channel
  uncertainty and probability summation.
\newblock {\em Vision Research}, 40(22):3121--3144, 2000.

\bibitem{Whitehill2016}
J.~Whitehill.
\newblock Exploiting an oracle that reports {AUC} scores in machine learning
  contests.
\newblock In {\em AAAI}, pages 1345--1351, 2016.

\bibitem{zheng2015toward}
W.~Zheng.
\newblock Toward a better understanding of leaderboard.
\newblock {\em arXiv preprint arXiv:1510.03349}, 2015.

\end{thebibliography}
\end{document}